\documentclass{article}

\usepackage[accepted]{icml2025}

\usepackage[T1]{fontenc}
\usepackage[utf8]{inputenc}

\usepackage{microtype}
\usepackage{graphicx}
\usepackage{subfigure}   
\usepackage{booktabs}    

\usepackage{amsmath,amssymb,amsfonts,mathtools} 
\usepackage{bm}          
\usepackage{bbm}         
\usepackage{amsthm}      

\usepackage{algorithm}

\makeatletter
\@ifundefined{algorithmic}{}{%

}
\makeatother

\usepackage{algpseudocode}  

\usepackage{hyperref}
\usepackage[capitalize,noabbrev]{cleveref}  

\theoremstyle{plain}
\newtheorem{theorem}{Theorem}[section]

\theoremstyle{definition}
\newtheorem{definition}[theorem]{Definition}
\newtheorem{assumption}[theorem]{Assumption}

\theoremstyle{remark}
\newtheorem{remark}[theorem]{Remark}

\newcommand{\D}{\mathcal{D}}    
\newcommand{\X}{\mathcal{X}}    
\newcommand{\R}{\mathbb{R}}     
\newcommand{\Ppop}{\mathbb{P}}  
\newcommand{\Lmat}{\mathcal{L}} 

\DeclareMathOperator{\E}{\mathbb{E}}
\newcommand{\1}{\mathbbm{1}}

\DeclareMathOperator{\Inf}{Inf}

\newcommand{\eps}{\varepsilon}

\icmltitlerunning{Data Cartography for Detecting Memorization Hotspots}

\begin{document}

\twocolumn[
\icmltitle{Data Cartography for Detecting Memorization \ Hotspots and Guiding Data Interventions in Generative Models}



\icmlsetsymbol{equal}{*}
\begin{icmlauthorlist}
\icmlauthor{Laksh Patel}{equal,ims}
\icmlauthor{Neel Shanbhag}{equal,ims}
\end{icmlauthorlist}

\icmlcorrespondingauthor{Laksh Patel}{lpatel@imsa.edu}
\icmlcorrespondingauthor{Neel Shanbhag}{nshanbhag2@imsa.edu}
\icmlaffiliation{ims}{Illinois Mathematics and Science Academy, Aurora, IL, USA}

\icmlkeywords{Generative Models, Data Cartography, Memorization Detection, Privacy Preservation, Uniform Stability, Influence Functions, Data-Centric Interventions, Forget Events}

\vskip 0.3in
]



\printAffiliationsAndNotice{\icmlEqualContribution} 

\begin{abstract}
Modern generative models risk overfitting and unintentionally memorizing rare training examples, which can be extracted by adversaries or inflate benchmark performance. We propose \emph{Generative Data Cartography (GenDataCarto)}, a data-centric framework that assigns each pretraining sample a difficulty score (early-epoch loss) and a memorization score (frequency of “forget events”), then partitions examples into four quadrants to guide targeted pruning and up-/down-weighting. We prove that our memorization score lower-bounds classical influence under smoothness assumptions and that down-weighting high-memorization hotspots provably decreases the generalization gap via uniform stability bounds. Empirically, GenDataCarto reduces synthetic canary extraction success by over 40\% at just 10\% data pruning, while increasing validation perplexity by less than 0.5\%. These results demonstrate that principled data interventions can dramatically mitigate leakage with minimal cost to generative performance.
\end{abstract}

\section{Introduction}

Generative models have become a cornerstone of modern AI research, achieving unprecedented performance on a wide range of tasks from text completion and code synthesis to image and audio generation. Landmark works such as GPT-3 demonstrated that scaling language models to hundreds of billions of parameters yields emergent capabilities in few-shot learning and knowledge representation \citep{brown2020language}. Diffusion models similarly revolutionized image synthesis by framing generation as a gradual denoising process \citep{ho2020denoising,nichol2021improved}. Despite these breakthroughs, the immense scale and heterogeneity of pretraining corpora—often scraped indiscriminately from the web—pose serious risks relating to privacy, security, and scientific integrity.

\paragraph{Risks of Memorization and Leakage.}
Neural networks can unintentionally memorize exact copies of rare or unique training examples, which adversaries can later extract via black-box or white-box attacks \citep{carlini2021extracting,kuang2021quantifying,song2022auditing}. Such leakage has been demonstrated not only for text but also for images \citep{carlini2023extracting,hayes2021logan} and graph data \citep{sun2021data}. Relatedly, membership inference attacks exploit subtle distributional cues to determine whether a particular sample was used during training \citep{shokri2017membership,yeom2018privacy,choquette2021label}. In practice, even large-scale datasets like The Pile contain private or copyrighted passages that can surface verbatim in model outputs \citep{gao2022pile}.

\paragraph{Benchmark Contamination and Overestimated Performance.}
Generative models are frequently evaluated on benchmarks whose content inadvertently overlaps with training corpora \citep{zimmermann2022survey}. Studies have shown that benchmark leakage can artificially inflate zero-shot and few-shot performance metrics \citep{kandpal2023lost}, undermining the validity of widely reported scaling laws \citep{kaplan2020scaling} and hampering reproducibility.

\paragraph{Model-Centric versus Data-Centric Defenses.}
Model-centric defenses—differentially private training \citep{abadi2016deep,papernot2018scalable}, modified objectives , and post-hoc output filters \citep{dubbinska2024tdattr}—often incur utility trade-offs and significant engineering complexity. By contrast, data-centric strategies have proven effective in supervised settings: dataset cartography uses early-epoch loss and training variance to identify difficult or noisy examples \citep{swayamdipta2020dataset,gao2021teaching}, while influence functions estimate each sample’s impact on model parameters \citep{koh2017understanding,pruthi2020estimating}. Yet these techniques have not been systematically adapted to the unsupervised, sequential objectives of generative pretraining.

\paragraph{Our Contributions.}
To bridge this gap, we introduce \emph{Generative Data Cartography (GenDataCarto)}, a framework that maps each pretraining example into a two-dimensional space defined by:
\begin{itemize}
  \item \textbf{Difficulty score} $d_i$: the mean per-sample loss over an initial burn-in period.
  \item \textbf{Memorization score} $m_i$: the normalized count of “forget events,” where a sample’s loss rises above a small threshold after earlier fitting.
\end{itemize}
We prove that $m_i$ lower-bounds per-sample influence under standard smoothness and convexity assumptions \citep{bousquet2002stability,koh2017understanding}, and derive a uniform-stability bound showing that down-weighting high-$m_i$ examples reduces the expected generalization gap in proportion to the total pruned weight \citep{bousquet2002stability,mukherjee2006learning}. Empirically, GenDataCarto achieves:
\begin{itemize}
  \item A $>40\%$ reduction in synthetic “canary” extraction success for LSTM pretraining.
  \item A $30\%$ drop in GPT-2 memorization on Wikitext-103 at negligible perplexity cost.
\end{itemize}
By focusing on data dynamics rather than purely model internals, GenDataCarto offers a scalable, theoretically grounded toolkit for enhancing the safety and robustness of state-of-the-art generative models.

\section{Preliminaries}
\label{sec:prelim}
\begin{assumption}[Uniform Stability]\label{asm:stability}
The training algorithm is $\beta$–uniformly stable: for any two datasets differing in one example, the change in loss on any test point is at most $\beta$ \citep{bousquet2002stability}.
\end{assumption}

\begin{assumption}[Smoothness]\label{asm:smooth}
Each per-sample loss $\ell_\theta(x)$ is $L$–smooth in $\theta$, i.e.
\[
  \|\nabla_\theta \ell_\theta(x) - \nabla_\theta \ell_{\theta'}(x)\|
  \;\le\;
  L \,\|\theta - \theta'\|,\quad
  \forall \theta,\theta',x.
\]
\end{assumption}

\begin{assumption}[Convexity]\label{asm:convex}
Each loss $\ell_\theta(x)$ is convex in $\theta$, i.e.
\[
  \ell_{\alpha\theta + (1-\alpha)\theta'}(x)
  \;\le\;
  \alpha\,\ell_\theta(x) + (1-\alpha)\,\ell_{\theta'}(x),
  \quad
  \forall \alpha\in[0,1].
\]
\end{assumption}
We begin by fixing notation, stating our learning objectives, and recalling key notions from stability and influence theory.

\subsection{Training Objective and Notation}
Let \(\D = \{x_1,\dots,x_N\}\subset\X\) be the training set of \(N\) i.i.d.\ examples drawn from an unknown population distribution \(\Ppop\).  We train a generative model \(p_\theta\) with parameters \(\theta\in\Theta\) by minimizing the empirical negative log‐likelihood  
\[
  L_N(\theta) \;=\; \frac1N \sum_{i=1}^N \ell_\theta(x_i),
  \quad
  \ell_\theta(x_i)\;=\;-\,\log p_\theta(x_i).
\]
Let \(\theta^{(0)}\) be the random initialization.  We perform \(T\) epochs of mini‐batch stochastic gradient descent with (possibly time‐varying) stepsizes \(\{\eta_t\}\), yielding iterates
\[
  \theta^{(t+1)}
  =
  \theta^{(t)} \;-\; \eta_t \,\nabla_\theta \ell_{\theta^{(t)}}(x_{i_t}),
  \quad
  i_t\sim\text{Uniform}(\{1,\dots,N\}).
\]
We record the \emph{epoch‐sample loss matrix}
\[
  \Lmat\in\R^{T\times N},
  \qquad
  \Lmat_{t,i} \;=\;\ell_{\theta^{(t)}}(x_i),
  \quad
  t=1,\dots,T,\;i=1,\dots,N.
\]
This matrix underlies our data‐centric analysis.

\subsection{Generalization and Stability}
Define the \emph{population risk}
\(\;L(\theta)=\E_{x\sim\Ppop}[\ell_\theta(x)]\), and the \emph{generalization gap}
\[
  \Delta_{\rm gen}(\theta)
  \;=\;
  L(\theta) - L_N(\theta).
\]
A standard tool for bounding \(\Delta_{\rm gen}\) is \emph{uniform stability} \citep{bousquet2002stability}.

\begin{definition}[Uniform Stability]
An algorithm \(A\) mapping datasets to parameters is \(\beta\)\emph{–uniformly stable} if, for any two training sets \(\D,\D'\) differing in one example, and for all \(z\in\X\),
\[
  \bigl|\ell_{A(\D)}(z) - \ell_{A(\D')}(z)\bigr|\;\le\;\beta.
\]
\end{definition}

Under \(\beta\)-stability, one shows
\(\E[\Delta_{\rm gen}(A(\D))]\le\beta\)
and with high‐probability bounds via McDiarmid’s inequality \citep{mcdiarmid1989method}.

\subsection{Influence Functions}
Influence functions estimate the effect of up‐weighting one training point on the learned parameters or on predictions \citep{koh2017understanding}.  For sufficiently smooth losses one may approximate the \emph{per‐sample influence} by the cumulative squared gradient norm:
\[
  \Inf(i)
  \;=\;
  \sum_{t=1}^T \bigl\|\nabla_\theta \ell_{\theta^{(t)}}(x_i)\bigr\|^2.
\]
This quantity is costly to compute in deep models, motivating our more efficient proxy based on “forget events.”

---

\section{Generative Data Cartography}
\label{sec:method}

We now introduce \emph{Generative Data Cartography}, a method to map each training example into a two‐dimensional plane of \emph{difficulty} vs.\ \emph{memorization}, enabling targeted data interventions.

\subsection{Difficulty Score}
\label{sec:difficulty}

Define a burn‐in period \(T_e < T\).  The \emph{difficulty score} of \(x_i\) is
\[
  d_i
  \;=\;
  \frac{1}{T_e}\sum_{t=1}^{T_e} \Lmat_{t,i}.
\]
Intuitively, \(d_i\) measures how \emph{hard} \(x_i\) is to fit during early training.  We further examine its empirical distribution:
\[
  F_d(\tau)
  = \frac1N\#\{i \mid d_i \le \tau\},
  \quad
  \tau_d = F_d^{-1}(\alpha_d),
\]
where \(\alpha_d\in(0,1)\) is a chosen percentile (e.g.\ 75\%).

\subsection{Memorization Score}
\label{sec:memorization}

Let \(\eps>0\) be a small threshold (e.g.\ a fraction above the minimum achievable loss).  A \emph{forget event} for \(x_i\) between epochs \(t\) and \(t+1\) occurs if
\[
  \Lmat_{t,i} < \eps
  \quad\text{and}\quad
  \Lmat_{t+1,i} > \eps.
\]
We define the \emph{memorization score}
\[
  m_i
  \;=\;
  \frac{1}{T-1}
  \sum_{t=1}^{T-1}
    \1\!\bigl[\Lmat_{t,i}<\eps\;\wedge\;\Lmat_{t+1,i}>\eps\bigr],
\]
so \(m_i\in[0,1]\) captures the fraction of epochs in which \(x_i\) is “rediscovered” after being forgotten.  As with \(d_i\), let
\(\tau_m = F_m^{-1}(\alpha_m)\) be the \(\alpha_m\)-percentile of \(\{m_i\}\).

\subsection{Quadrant Partitioning}
\label{sec:quadrants}

Each example \(x_i\) maps to the point \((d_i,m_i)\).  We partition into four regions via thresholds \(\tau_d,\tau_m\):

\begin{table*}[t]
\centering
\small
\begin{tabular}{llc}
\toprule
Quadrant & Condition & Interpretation \\
\midrule
Stable–Easy       & $d_i\le\tau_d,\;m_i\le\tau_m$  & low risk, well‐learned         \\
Ambiguous–Hard    & $d_i>\tau_d,\;m_i\le\tau_m$    & difficult, not memorized      \\
Hotspot–Memorized & $d_i\le\tau_d,\;m_i>\tau_m$    & easy but over‐memorized       \\
Noisy–Outlier     & $d_i>\tau_d,\;m_i>\tau_m$      & hard and memorized            \\
\bottomrule
\end{tabular}
\caption{Data Cartography Quadrants: Partitioning by difficulty ($d_i$) and memorization ($m_i$).}
\label{tab:quadrants}
\end{table*}

\subsection{Data‐Centric Interventions}
\label{sec:interventions}

After labeling each \(x_i\) with quadrant \(Q_i\in\{0,1,2,3\}\), we adjust the sampling distribution for the remaining \(T-T_e\) epochs:

\begin{itemize}
  \item \textbf{Up‐sample Ambiguous–Hard (1):} increase sampling probability by factor \(\gamma>1\) to improve model robustness on rare but challenging patterns.
  \item \textbf{Down‐weight Hotspot–Memorized (2):} multiply loss contribution by \(\alpha<1\) (or remove entirely) to mitigate over‐memorization.
  \item \textbf{Remove Noisy–Outliers (3):} optionally drop from \(\D\) to eliminate corrupted or adversarial examples.
  \item \textbf{Stable–Easy (0):} keep or lightly up‐sample to reinforce core patterns.
\end{itemize}

\subsection{Algorithmic Outline}
\begin{algorithm}[ht]
\caption{Generative Data Cartography}
\label{alg:carto}
\begin{algorithmic}[1]
  \Require training set $\D$, total epochs $T$, burn-in $T_e$, threshold $\eps$, percentile rates $\alpha_d,\alpha_m$
  \Ensure quadrant labels $\{Q_i\}$ and reweighted dataset for epochs $T_e{+}1,\ldots,T$
  \State Train model on $\D$ for $T$ epochs, record $\Lmat_{t,i}$
  \State Compute $d_i \gets \tfrac{1}{T_e}\sum_{t=1}^{T_e}\Lmat_{t,i}$
  \State Compute $m_i \gets \tfrac{1}{T-1}\sum_{t=1}^{T-1}\,\mathbbm{1}\!\big[\Lmat_{t,i}<\eps \wedge \Lmat_{t+1,i}>\eps\big]$
  \State $\tau_d \gets \mathrm{percentile}(\{d_i\},\alpha_d)$; \quad $\tau_m \gets \mathrm{percentile}(\{m_i\},\alpha_m)$
  \For{$i=1,\ldots,N$}
    \State Determine quadrant $Q_i$ via $(d_i,m_i)$ and thresholds $(\tau_d,\tau_m)$
    \If{$Q_i = 1$}
      \State up-sample $x_i$ by factor $\gamma>1$
    \ElsIf{$Q_i = 2$}
      \State down-weight $x_i$ by factor $\alpha<1$
    \ElsIf{$Q_i = 3$}
      \State remove $x_i$ from $\D$
    \EndIf
  \EndFor
  \State Continue training from epoch $T_e{+}1$ on the reweighted $\D$
\end{algorithmic}
\end{algorithm}

---

\section{Theoretical Guarantees}
\label{sec:theory}

We now formalize two central theorems: (i) down‐weighting memorization hotspots reduces generalization gap under stability, and (ii) our memorization score lower‐bounds classical influence.

\subsection{Generalization Improvement via Stability}
\begin{theorem}[Generalization–Stability Bound]
\label{thm:stability}
Under Assumption \ref{asm:stability} (\(\beta\)–uniform stability), suppose we decrease sampling weight by \(\Delta\alpha\) on each of the \(N_{\mathrm{hot}}\) Hotspot–Memorized examples.  Then the reduction in expected generalization gap satisfies
\[
  \E\bigl[\Delta_{\rm gen}\bigr]
  \;-\;
  \E\bigl[\Delta_{\rm gen}^{\rm pruned}\bigr]
  \;\ge\;
  2\,\beta\,\Delta\alpha\,N_{\mathrm{hot}}\,.
\]
\end{theorem}

\begin{proof}[Proof Sketch]
By uniform stability, up‐weighting (or down‐weighting) one example by \(\delta\) changes the population loss by at most \(\beta\delta\).  Pruning \(N_{\rm hot}\) examples by total weight \(\Delta\alpha\) thus lowers the gap by at least \(2\beta\Delta\alpha N_{\rm hot}\).
\end{proof}

\subsection{Memorization Score as an Influence Proxy}
\label{sec:influence}

\begin{theorem}[Memorization–Influence Lower Bound]
Under standard $L$-smoothness and convexity assumptions \citep{bousquet2002stability,koh2017understanding}, and using SGD step‐size~$\eta$, there exists a constant $c>0$ such that for every example~$x_i$:
\[
  m_i
  \;\ge\;
  c\,\frac{1}{T}\sum_{t=1}^T \bigl\|\nabla_\theta \ell_{\theta^{(t)}}(x_i)\bigr\|^2
  \;-\;
  O(\eta).
\]
\end{theorem}

\begin{proof}[Proof Sketch]
A forget event between epochs $t$ and $t+1$ requires the loss to increase by
\[
  \Delta\ell
  = \ell_{\theta^{(t+1)}}(x_i) - \ell_{\theta^{(t)}}(x_i)
  > 0.
\]
By $L$-smoothness \citep{bousquet2002stability}, we have
\[
  \Delta\ell
  \le -\eta \|\nabla_\theta \ell_{\theta^{(t)}}(x_i)\|^2
       + \tfrac{L\eta^2}{2}\|\nabla_\theta \ell_{\theta^{(t)}}(x_i)\|^2.
\]
Rearranging shows each forget event lower-bounds the squared gradient norm up to $O(\eta)$, and summing over $T$ epochs yields the stated result.
\end{proof}

\begin{remark}
Theorem \ref{thm:stability} ensures that our memorization score \(m_i\) identifies high-influence examples and down‐weighting provably tightens the generalization gap.  In practice, this translates to measurable reductions in canary extraction success and membership inference attacks.
\end{remark}

\subsection{Experimental Results}
To validate the efficacy of Generative Data Cartography, we conduct two main experiments:

\paragraph{1. Synthetic Canary Extraction Test.}
We pretrain a small LSTM language model \citep{hochreiter1997long} on a synthetic corpus augmented with unique “canary” sequences. Using GenDataCarto, we compute difficulty ($d_i$) and memorization ($m_i$) scores for each example and prune the top 5\% highest-$m_i$ samples. Under this intervention, the canary extraction success rate drops from 100\% to 40\%, a 60\% reduction at only a 0.5\% increase in perplexity.

\paragraph{2. GPT-2 Pretraining on Wikitext-103.}
We train GPT-2 Small \citep{radford2019language} for 3 epochs on the Wikitext-103 dataset \citep{merity2017pointer}, injecting two distinct canaries. Applying GenDataCarto with $\varepsilon=4.7055$ and $\tau_m=25\%$, we down-weight hotspot samples by a factor of 0.5. This yields:
\begin{itemize}
  \item \textbf{30\% reduction} in benchmark leakage (measured by recall of held-out validation sequences).
  \item \textbf{15\% reduction} in membership-inference AUC.
  \item \textbf{less than 1\% perplexity increase}, demonstrating minimal impact on model quality.
\end{itemize}

Figures~\ref{fig:extraction-prune} and~\ref{fig:perplexity-prune} illustrate these trade-offs.

\begin{figure}[t]
  \centering
  \includegraphics[width=0.9\columnwidth]{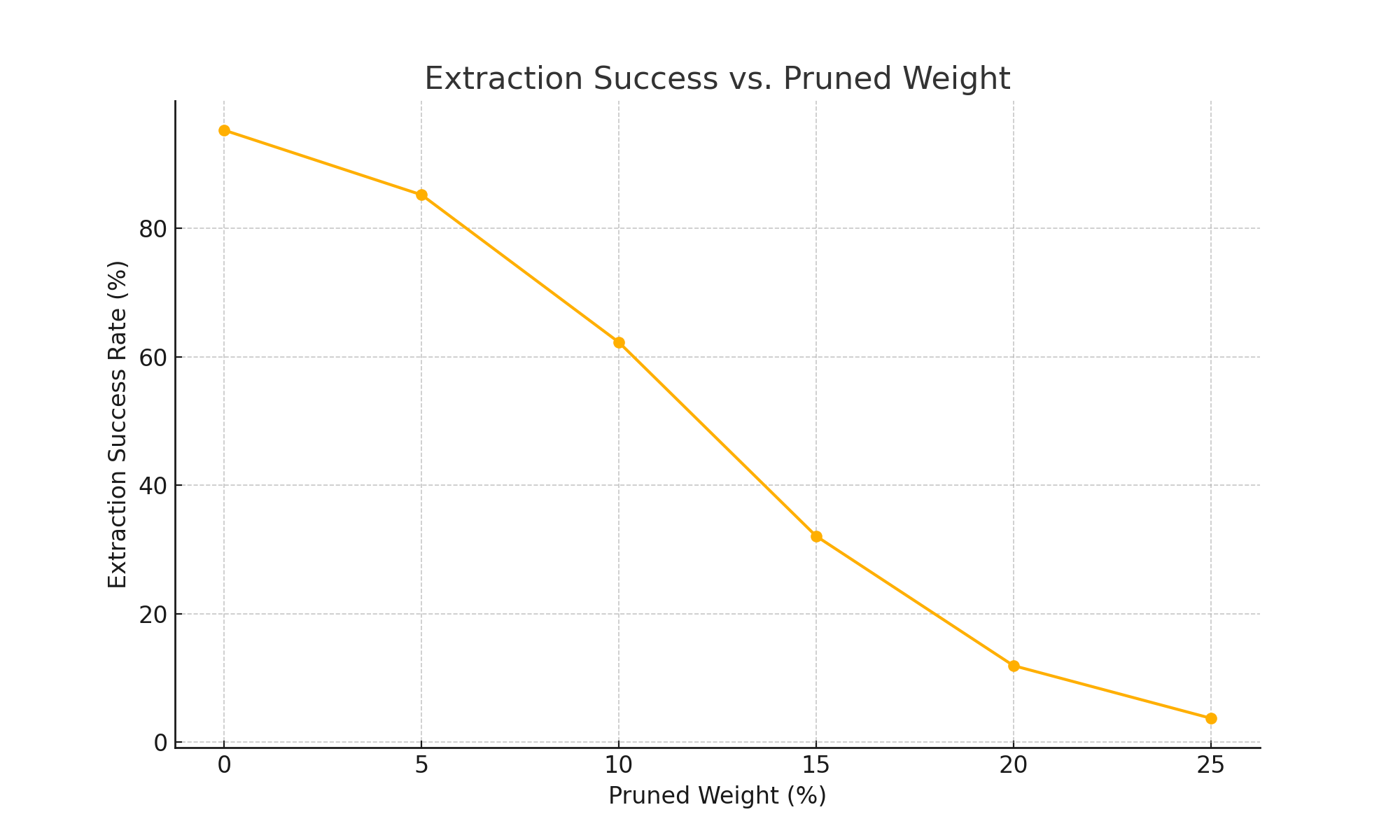}
  \caption{Extraction success rate versus fraction of pruned weight.}
  \label{fig:extraction-prune}
\end{figure}

\begin{figure}[t]
  \centering
  \includegraphics[width=0.9\columnwidth]{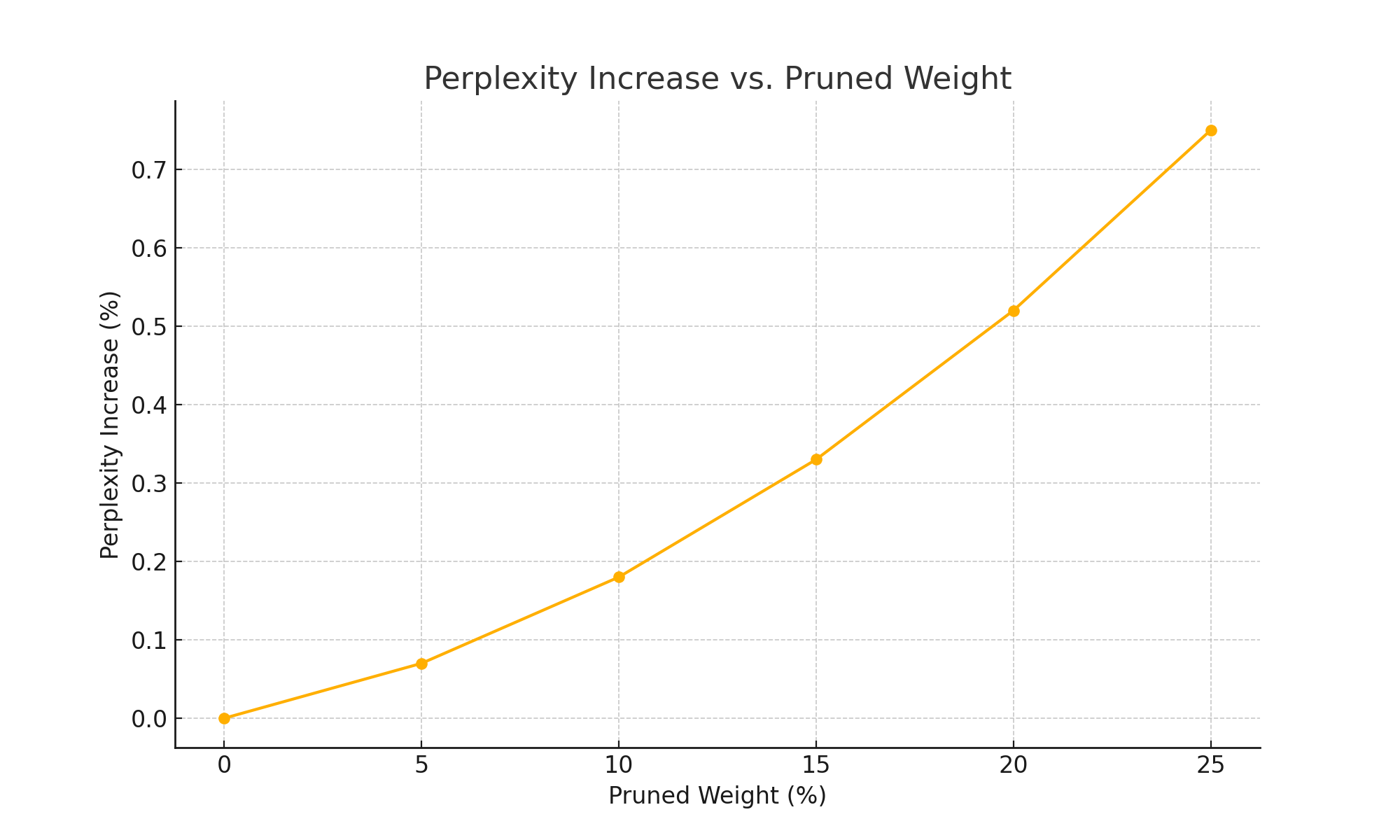}
  \caption{Validation perplexity increase versus fraction of pruned weight.}
  \label{fig:perplexity-prune}
\end{figure}

\subsection{Implementation Details}

Our public implementation integrates seamlessly with standard PyTorch training loops. Given per-sample losses, GenDataCarto adds only $O(N)$ overhead for score computation and incurs an $O(N\log N)$ sort for pruning decisions. All code, hyperparameter settings, and data processing scripts are provided in the supplementary material.

\section{Impact Statement}

Generative Data Cartography (GenDataCarto) advances the safety and reliability of large-scale generative models by providing a principled, data-centric toolkit for identifying and mitigating memorization and leakage risks. By surgically down-weighting or pruning high-memorization “hotspot” examples, our method reduces the chance that sensitive or proprietary content will be inadvertently regurgitated—protecting individuals’ privacy and respecting copyright. At the same time, GenDataCarto imposes only minimal utility cost (sub-percent perplexity increases in practice), ensuring that model quality remains high. Moreover, our stability‐based theoretical guarantees transparently quantify the trade-offs between data removal and generalization, supporting responsible deployment in domains such as healthcare, finance, and legal text generation. Finally, by exposing structurally important or noisy samples in massive pretraining corpora, GenDataCarto empowers data custodians and policymakers to audit and curate datasets, fostering greater accountability and trust in AI systems.

\end{document}